\documentclass[preprint,12pt]{elsarticle}
\makeatletter
\def\ps@pprintTitle{%
 \let\@oddhead\@empty
 \let\@evenhead\@empty
 \def\@oddfoot{\centerline{\thepage}}%
 \let\@evenfoot\@oddfoot}
\makeatother
\usepackage[margin=1in]{geometry}

\newcommand*\patchAmsMathEnvironmentForLineno[1]{%
  \expandafter\let\csname old#1\expandafter\endcsname\csname #1\endcsname
  \expandafter\let\csname oldend#1\expandafter\endcsname\csname end#1\endcsname
  \renewenvironment{#1}%
     {\linenomath\csname old#1\endcsname}%
     {\csname oldend#1\endcsname\endlinenomath}}%
\newcommand*\patchBothAmsMathEnvironmentsForLineno[1]{%
  \patchAmsMathEnvironmentForLineno{#1}%
  \patchAmsMathEnvironmentForLineno{#1*}}%
\AtBeginDocument{%
\patchBothAmsMathEnvironmentsForLineno{equation}%
\patchBothAmsMathEnvironmentsForLineno{align}%
\patchBothAmsMathEnvironmentsForLineno{flalign}%
\patchBothAmsMathEnvironmentsForLineno{alignat}%
\patchBothAmsMathEnvironmentsForLineno{gather}%
\patchBothAmsMathEnvironmentsForLineno{multline}%
}

\usepackage{graphicx} % more modern
\graphicspath{{figures/}} % Location v of the graphics files
\usepackage{multirow}

\usepackage{array}
\newcolumntype{L}[1]{>{\raggedright\let\newline\\\arraybackslash\hspace{0pt}}m{#1}}
\newcolumntype{C}[1]{>{\centering\let\newline\\\arraybackslash\hspace{0pt}}m{#1}}
\newcolumntype{R}[1]{>{\raggedleft\let\newline\\\arraybackslash\hspace{0pt}}m{#1}}

\usepackage{subfigure}
\usepackage{natbib}
\usepackage{algorithm}
\usepackage{algorithmic}

\usepackage{epsfig}
\usepackage{amsthm}
\usepackage{amsmath}
\usepackage{amssymb}
\usepackage{epstopdf}
\usepackage{booktabs}
\usepackage{color}
\usepackage{hyperref}
\usepackage[justification=centering]{caption}
\usepackage{longtable}

\usepackage{breqn}

\newcounter{thm_counter}
\newtheorem{lemma}[thm_counter]{Lemma}%[Lemma]

\def\1{{\bf{1}}}
\def\0{{\bf{0}}}

\def\w{{\bf w}}
\def\u{{\bf u}}
\def\v{{\bf v}}

\def\x{{\bf x}}
\def\y{{\bf y}}

\newcommand{\beq}{\begin{equation}}
\newcommand{\eeq}{\end{equation}}

\newcommand{\R}{\mathbb{R}}

\def\revise#1{\textcolor{black}{#1}}

\usepackage{lineno}

\begin{document}

\begin{frontmatter}


\title{Prognostics of Surgical Site Infections using Dynamic Health Data}



\author[add1]{Chuyang Ke}
\author[add2]{Yan Jin}
\author[add3]{Heather Evans}
\author[add4]{Bill Lober}
\author[add5]{Xiaoning Qian}
\author[add1]{Ji Liu}
\author[add2]{Shuai Huang}

\address[add1]{Department of Computer Science, University of Rochester}
\address[add2]{Department of Industrial $\&$ Systems Engineering, University of Washington}
\address[add3]{Department of Surgery, University of Washington}
\address[add4]{Department of Biomedical Informatics and Medical Education, University of Washington}
\address[add5]{Department of Electrical $\&$ Computer Engineering, Texas A$\&$M University}

\begin{abstract}
Surgical Site Infection (SSI) is a national priority in healthcare research. Much research attention has been attracted to develop better SSI risk prediction models. However, most of the existing SSI risk prediction models are built on static risk factors such as comorbidities and operative factors. In this paper, we investigate the use of the dynamic wound data for SSI risk prediction. There have been emerging mobile health (mHealth) tools that can closely monitor the patients and generate continuous measurements of many wound-related variables and other evolving clinical variables. Since existing prediction models of SSI have quite limited capacity to utilize the evolving clinical data, we develop the corresponding solution to equip these mHealth tools with decision-making capabilities for SSI prediction with a seamless assembly of several machine learning models to tackle the \revise{analytic challenges arising from the spatial-temporal data. The basic idea is to exploit the low-rank property of the spatial-temporal data via the bilinear formulation, and further enhance it with automatic missing data imputation by the matrix completion technique.} We derive efficient optimization algorithms to implement these models and demonstrate \revise{the superior performances of our new predictive model on a real-world dataset of SSI, compared to} a range of state-of-the-art methods.
\end{abstract}


\end{frontmatter}

%\linenumbers

\section{Introduction}
\label{S:1}

Surgical Site Infection (SSI) is a national priority in healthcare research~\cite{Mahmoud2009,Kirkland1999,SangerRamshorstMercanEtAl2016}. It occurs in 3-5\% of all surgical patients, and up to 33\% of patients undergoing abdominal surgery \cite{ROSSINI2013,Krieger2011}. More than 500,000 cases are estimated to occur in the US annually, resulting in additional costs of up to \$20,000 per infection. It also results in worse health outcomes for patients, such as prolonged length of hospital stay, increased mortality, and compromised health-related quality of life \cite{Perencevich2009,Dipiro1998,Zimlichman2013}. SSI is overall the most costly healthcare-associated infection, yet many of its associated costs are non-reimbursable. Surveillance methods have been invented since the early 1980s to provide appropriate data (such as risk indicators, clinical prediction rules) to surgeons to monitor how changes in practice can impact SSI occurrence \cite{Horan1992,Haley1985,PO2012}. Many of these surveillance systems follow the standard guidelines established in 1992 by the American Centers for Disease Control and Prevention (CDC)'s National Nosocomial Infections Surveillance (NNIS) system, and rely on volunteer surgical wards from public or private hospitals that routinely collect nosocomial infection data. Although these surveillance systems are continuously updated with successive risk indicators being included, they are mainly based on individual or very few risk factors. Furthermore, the quality of post-discharge data is poor, as active surveillance is costly and rarely performed prospectively, relying on retrospective review of clinical documentation. Recently, with the rapid advances of sensing and information technologies, new tools have been developed to provide patients the means to monitor their conditions and share this data with their providers. One example is the mobile health (mHealth) tool that we have developed to enable patient-initiated monitoring of surgical wounds and improve patient-provider communication after hospital discharge. Self-reported symptoms of pain, body temperature, and wound features, as well as patient- or caregiver-generated images of the wound may be acquired to assess the evolving condition of the wound and the likelihood of SSI.

While these emerging mHealth tools have provided unprecedented capacity to closely monitor patients for signs and symptoms of SSI and measure the evolution of the wound-related variables and other clinical factors, existing prognostic models of SSI have not been able to take advantage of this rich and accumulating information for SSI prediction. Prognostic models of SSI (\cite{lawson2013risk,berger2013development,van2013surgical,ho2011differing,saunders2014improving}) in the literature have three main limitations. First, many of these models were based on medical knowledge or heuristics, not data-driven by nature. Thus, their utility is limited on some specific phenotypes or particular time windows of the SSI progression process. Second, these models only predict whether or not the patient will develop SSI, but do not predict when. Third, most existing models only incorporate static variables known as of the end of the operation, e.g., demographics, pre-operative laboratory results, comorbidities, and operative factors. These models do not incorporate the continuous observations of the patients and their wound.

\begin{figure}[!t]
\centering
\includegraphics[width=0.5\linewidth]{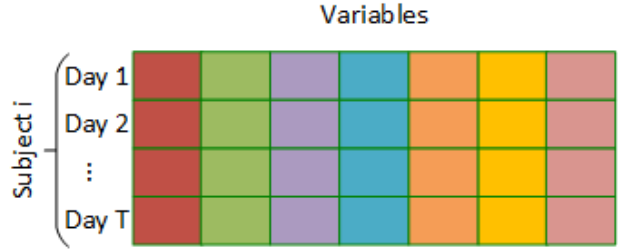}
\caption{An illustration of the dynamic data from a subject as a spatial-temporal data matrix}
\label{fig:sample}
\end{figure}

To overcome the limitations of existing SSI prediction models, especially when utilizing the dynamic data collected by mHealth tools \revise{(e.g., an illustration of such dynamic data as a spatial-temporal data matrix is shown in Figure~\ref{fig:sample}), we develop a solution that is specifically designed to address the unique analytic challenges from dynamic mHealth data. The collected dynamic data includes continuous measurements on a range of wound variables and clinical factors, for which we form a spatial-temporal data matrix for each individual (Figure~\ref{fig:sample}). Rather than predicting the development of SSI based on a snapshot of these factors or the static variables known as of the end of the surgical operation, we propose a learning formulation to predict the time to the development of SSI directly based on these spatial-temporal matrices. In our experience with real-world applications, one common data challenge is that there are many patients who do not develop SSI during observational period} (in this paper we name them as ``censored samples''). What is worse is that, the number of censored samples is usually greater than the number of samples that have observed outcomes (\revise{in this paper we name them as ``complete samples''}). Thus, a learning formulation that can fuse both types of samples is needed. Another data challenge is, as a well-known phenomenon in many healthcare applications, we encounter many missing values that result in a challenge for the learning formulation. We therefore propose a matrix completion approach to overcome this challenge and derive a convex optimization formulation to solve it. Further, with the spatial-temporal matrix representation for the dynamic mHealth data as the input of the prediction model, we investigate the use of a bilinear formalism to reduce the dimensionality of the prediction model. Our main contributions include:
\begin{itemize}
	\item  Developed a flexible learning formulation that can convert the dynamic clinical signals of an individual (represented as the ``spatial-temporal matrix") into accurate prognostics of SSI onset, even when some data have censored outcome;
	\item  Proposed the use of matrix completion to mitigate the missing data problem that has been found ubiquitous in many healthcare applications;
	\item  Developed efficient algorithms to implement the machine learning models; specific optimization strategies were developed to ensure the feasibility and robustness of the algorithms;
	\item  Conducted extensive numerical studies on a real-world dataset (\revise{that includes 860 patients observed in a time window ranging from 2 to 21 days while 167 developed SSI}), which generated real-world experience and empirical evidence for using the proposed method with dynamic mHealth data.
\end{itemize}

The remainder of the paper is organized as follows: In Section 2, a brief review of related work will be provided. In Section 3, our method will be presented with detailed explanations of the formulation, together with the computational algorithm for solving the proposed formulation. The real-world implementation of this method will be presented and discussed in Section 4. Conclusions will be drawn in Section 5.

\section{Related Work}
\label{S:2}
Our work is closely related to the following two areas: prediction models of SSI risk and machine learning methods that are related to the analysis of spatial-temporal data.

\textbf{Risk prediction models of SSI}: Many current risk prediction models for SSI are built upon expert opinions. For instance, the Ventral Hernia Working Group (VHWG) has categorized patients into 4 grades: low risk, comorbid, potentially contaminated, and infected. A grading system has thereafter been developed to predict the risk of SSI, based on expert opinions but not directly on patient data. This model has later been modified by researchers, including the 3-tier system developed in \cite{kanters2012modified}, which are still knowledge-driven rather than being data-driven. On the other hand, many data-driven risk scores for SSI have been developed over the years, such as the ones in \cite{lawson2013risk,berger2013development,van2013surgical,ho2011differing,saunders2014improving}. Among these models, the National Nosocomial Infection Surveillance (NNIS) Risk Index is a well-accepted risk-assessment tool, but it only includes three predictors. Thus, it has limited capacity to utilize the much more information that we could collect nowadays for SSI prediction. For more complex models, such as the model developed in \cite{lawson2013risk}, a hierarchical multivariate logistic regression model was used to discriminate SSI from non-SSI patients based on risk factors such as some operative variables, preoperative clinical severity, risk factors, comorbidities, and other variables related to the hospitalization procedure. However, these existing models build on the static measurements of some selected risk factors rather than the dynamic post-operative symptom and wound observations. \revise{ Recently, \cite{SangerRamshorstMercanEtAl2016} investigated the use of the data of ``last five days'' (defined as the last 5 days  prior to SSI onset for SSI patients or the last 5 days for non-SSI patients) for predicting SSI onset, which showed better performance than using only cross-sectional measurements. Comparing with \cite{SangerRamshorstMercanEtAl2016}, our study focuses more on the methodological issues of how to mitigate the statistical challenges in order to automate the process of using the incomplete dynamic data to build and select the best prediction model. More related works that could be applied on a range of disease contexts can be found in \cite{Peek2014}, which has covered a wide range of prognostic models.}
%\textcolor{blue}{Shuai to Yan: Yan, the format of some references seems messy}

\textbf{Machine learning methods}: Matrix completion \cite{candes2009exact, cai2010singular, ji2010robust, liu2013tensor, signoretto2011tensor} is an important approach to estimate missing elements in a data matrix, \revise{without making strong assumptions on data missing mechanisms}. The authors in \cite{candes2009exact} proved that the low-rank matrix completion technique only needs $O(mr\log^2 m)$ to exactly recover the missing values for a $m\times n$ matrix with rank $r$, assuming $m\geq n$. A related formalism that can be used for building prediction models using spatial-temporal data could be the bilinear models. The bilinear model is a popular approach to capture information in two different dimensions, for example in text mining \cite{tenenbaum2000separating}, computer vision \cite{freeman1997learning}, and image processing \cite{olshausena2007bilinear}. \revise{A specific example is, for classification using brain image data or EEG signals, the input data for classification are represented as matrices with columns representing space indices and rows representing time indices. It is always possible to apply linear models by creating data vectors as stacking of the elements of the data matrix; however, such stacking approaches can not exploit the spatial and temporal structure in the data matrix and will inevitably lead to a high-dimensional weight vector in the linear model. Thus, we adopt the bilinear formulation to better exploit the data matrix structure to obtain a more parsimonious representation of the weight vector.}  For instance, Singular Value Decomposition~(SVD) is a typical bilinear model that identifies representative vectors (characterized as eigenvectors) in both dimensions to span the data space. A general framework was presented in~\cite{tenenbaum2000separating} for solving two-factor tasks using bilinear models, which can characterize factor interactions and employ efficient algorithms based on the singular value decomposition and expectation-maximization. Although not directly applicable here, existing matrix completion formulations and bilinear models provide a great resource for us to develop a systematic data analytic pipeline by building on these models and tackle the data challenges with the dynamic mHealth data for SSI prediction.

\section{Approach}
\label{S:3}
\subsection{An Overview of the Proposed Machine Learning System}
The overall goal of the proposed machine learning system is to predict the time to SSI onset using the spatial-temporal matrix data. There are several major data challenges that need to be overcome, such as the fusion of both complete samples and censored samples, missing values in the spatial-temporal matrices, and the complex two-way interaction structure of the spatial-temporal matrices.

To mitigate these challenges, Figure \ref{fig:flow} illustrates an overall \revise{scheme of the proposed machine learning system. It consists of a novel modeling formulation to fuse both complete samples and censored samples, in which bounded matrix completion is used to mitigate the missing data problem. The learning formulation is further equipped with model regularization techniques such as the bilinear formalism and sparse learning to mitigate the challenges from the high dimensionality of the spatial-temporal matrix and to exploit its inherent spatial and temporal structure}. In the following subsections, each of the main machine learning algorithms will be described in details. \\

\begin{figure}[h]
    \centering
    \includegraphics[width=\textwidth]{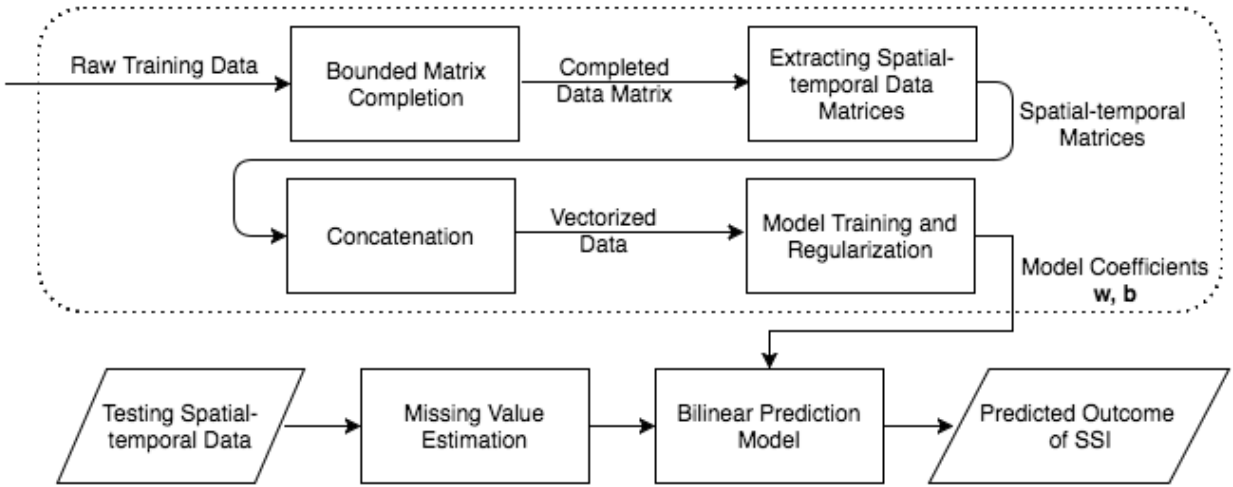}
    \caption{Algorithm flow}
    \label{fig:flow}
\end{figure}

\subsection{Learning Formulation for Data Fusion}
Denote the $i$-th spatial-temporal data sample in the training data as $\x_i \in \mathbb{R}^{T\times P}$, where $T$ denotes the length of observations and $P$ denotes the number of dynamic variables (such as those wound-related variables) that are collected. Note that here, the optimal length of observations $T$ is not a known parameter. Rather, it should be decided using data-driven methods, which will be discussed in Section 4. As we have mentioned, our training data consists of both complete samples and censored samples. We use $\Omega$ to denote the index set of complete samples, and use $\bar{\Omega}$ to denote the index set of censored samples. \revise{Missing data is referred to the missing elements in the matrix $\x_i$.} Further, we use $y_i$ to denote the time to develop SSI for sample $i$ if this sample is in $\Omega$, or the duration of observations if the sample $i$ is in $\bar{\Omega}$. \revise{For example, suppose that the sample $i$ is the subject $i$'s day 1 through day 5's measurements. If the subject $i$ developed SSI on day 7, which means $i\in\Omega$, $y_i = 7-5 = 2$. Otherwise, if the subject $i$ did not develop SSI in the observational period of 21 days ($i\notin\Omega$), $y_i = 21 - 5 = 16$.} Then, the following learning formulation is proposed to learn the prediction model based on both the complete and censored samples, under the assumption that there is no missing data in $\x_i$ and the prediction model is linear:

\begin{equation}
\begin{aligned}
\min_{\w \in \R^{T\times P}, b \in \R} &\quad
f(\w,b):=
\sum_{i\in\Omega} \frac{1}{2}(\langle \x_i,~\w\rangle+b-y_i)^2
\\
& +
\sum_{j\in\bar{\Omega}} \frac{\lambda}{2}(\min(0, \langle \x_j,~\w\rangle +b-y_j))^2. % \\%+ \phi \|\w\|_1\\
\end{aligned}
\label{eq:origin}
\end{equation}
\revise{Here, $\langle\textbf{x}_i, \textbf{w}\rangle$ denotes the the sum of entry-wise products of two matrices. Note that, the formulation consists of two components in the objective function. The first component, $\sum_{i\in\Omega} \frac{1}{2}(\langle \x_i,~\w\rangle+b-y_i)^2$, is the least-square loss function that penalizes the difference between the predicted values, i.e., $\langle \x_i,~\w\rangle+b$, with the real outcome data, i.e., $y_i$, from the complete samples. This is the same with classic linear regression models. The second component evaluates the squared hinge loss~\cite{bartlett2008classification}, which corresponds to the loss function for the censored samples. The rationale for the use of the squared hinge loss is that, if the predicted time to SSI onset is larger than the censored time $y_i$, we should not penalize the model; otherwise, we penalize the model. In this way, the censored samples can contribute to the learning of the prediction model in an unbiased way. To balance the contributions of these two types of samples, we further use a parameter, $\lambda$, in the formulation, i.e., larger $\lambda$ will allow the censored samples to give more influence on the model estimation. In practice, this parameter can be selected by model selection methods such as cross-validation to enable data-driven optimal model selection.}

\subsection{Bilinear Formulation for Spatial-Temporal Matrix}
Note that the model proposed above imposes no structure on the model parameters $\w$, ignoring the fact that the input data is a spatial-temporal matrix that has an inherent correlation structure in both dimensions. Here, we propose to adopt a bilinear formulation to capture the correlation structure within the spatial-temporal matrix. This is inspired by the success of the bilinear models in a range of applications such as \cite{tenenbaum2000separating, olshausena2007bilinear}. Specifically, here, by using the bilinear formulation, the prediction model can be rewritten as

\begin{equation}
\langle \x_i,~\w\rangle+b-y_i = \sum_{r=1}^R \u_r^\top \x \v_r+b-y_i ,
\end{equation}
where $\u\in \R^{T\times R}$ and $\v\in \R^{P\times R}$ with $\u_r$ and $\v_r$ denoting their corresponding $r$th column. Actually, the bilinear formulation can be interpreted from another point of view for dimension reduction by noting that $\sum_{r=1}^R \u_r^\top \x \v_r =  \langle \u\v^\top, \x \rangle$. In other words, to learn the bilinear prediction model characterized by $\u\in \R^{T\times R}$ and $\v\in \R^{P\times R}$, it is equivalent to learn the optimal solution of $\w$ if we restrict $\w$'s rank by $R$ in the formulation \eqref{eq:origin}. This leads to the following transformed learning formulation:

\begin{equation}
\begin{aligned}
\min_{\w \in \R^{T\times P}, b \in \R} &\quad
f(\w,b):=
\sum_{i\in\Omega} \frac{1}{2}(\langle \x_i,~\w\rangle+b-y_i)^2
\\
& +
\sum_{j\in\bar{\Omega}} \frac{\lambda}{2}(\min(0, \langle \x_j,~\w\rangle +b-y_j))^2 \\%+ \phi \|\w\|_1\\
\text{s.t.}& \quad \text{rank}(\w) \leq r.
\end{aligned}
\label{eq:main}
\end{equation}

Hence, besides the bilinear formalism, we note that this formulation can also be interpreted as a rank-regularized formulation. The low-rank constraint of $\w$ is a reasonable treatment here, since due to the potential spatial correlations (i.e., correlations between the variables) and temporal correlations (correlations between the time points), the spatial-temporal \revise{data inherently lie on a low-dimensional manifold. Therefore,  $\w$ has the corresponding low rank. Thus, the bilinear formulation induces the} rank-regularized formulation of $\w$.

\subsection{Preconditioning Reformulation and Optimization}
To solve the formulation in \eqref{eq:main}, first, we rewrite the formulation by vectorizing $\w$ and each spatial-temporal data $\x_i$. Particularly, we concatenate the rows of the matrix to convert the matrix to be a vector. Similarly, to generate \revise{the corresponding design matrices $X_{\Omega}$ and $X_{\bar{\Omega}}$ for the complete samples and censored samples respectively, a column vector can be constructed by concatenating the rows of a spatial-temporal matrix $\x_i$. A graphical illustration of this concatenation process will be presented in Figure ~\ref{fig:gen}. Also, $\y_{\cdot}$ denotes the corresponding vector of outcomes or censored outcomes. Then, we can rewrite \eqref{eq:main} as the following vectorized form:}
\begin{equation}
\begin{aligned}
\min_{\w \in \R^{T\times P}, b \in \R} \quad
&\frac{1}{2}\left\|X_\Omega^\top \text{vec}(\w) - \1 b-\y_\Omega\right\|^2
\\
+ &\frac{\lambda}{2}\left\|\min\left(\0, X_{\bar{\Omega}}^\top \text{vec}(\w) - \1 b - \y_{\bar{\Omega}}\right)\right\|^2 \\
\text{s.t.}& \quad \text{rank}(\w) \leq r,
\end{aligned}
\label{eq:main_2}
\end{equation}
\revise{where $\mathbf{1}$ denotes the vector whose entries are all one.} The objective function in \eqref{eq:main_2} is convex, and the projection onto the rank constraint has a closed form. It is tempting to apply the projected gradient descent (PGD) approach (\cite{nesterov2004introductory}) to solve this convex formulation. However, it would be inefficient if one directly applies PGD. The reason is that, due to the large spatial and temporal correlations in the spatial-temporal matrix, the optimization problem could be ill conditioned (recall the way by which we generated the matrix $X_{\Omega}$ and $X_{\bar{\Omega}}$). Thus, we propose the following preconditioning reformulation to enable the application of the PGD on solving for \eqref{eq:main_2}. \\
Let $\text{vec}(\hat{\w}) = (X_{\Omega} X_{\Omega}^\top)^{1\over 2}\text{vec}(\w)$. We can substitute $\text{vec}(\w)$ by $\text{vec}(\hat{\w})$ in the objective function in \eqref{eq:main_2}. The challenging issue is how to substitute $\text{vec}(\w)$ by $\text{vec}(\hat{\w})$ in the rank constraint. In fact, Lemma 1 indicates that $\text{rank}(\hat{\w}) = \text{rank}(\w)$ as long as $(X_{\Omega} X_{\Omega}^\top)^{1\over 2}$ is invertible, i.e., in other words, as long as $X_{\Omega}$ has full column rank.

\begin{lemma}
For any matrix $\w$, its rank remains the same after the vectorization and a full ranked linear transformation, that is, $\text{rank}(\hat{\w}) = \text{rank}(\w)$, where $A$ is any full rank square matrix and
\[
\revise{\text{vec}(\hat{\w}) = A\cdot\text{vec}(\w).}
\]
\end{lemma}
\begin{proof}
We first prove that $\text{range}(\hat{\w}) \subset \text{range}(\w)$. \revise{Note that the range of a matrix is defined as the span (set of all possible linear combinations) of its column vectors.} It is easy to see that any column of $\hat{\w}$ is the linear combination of all columns in $\w$. So the range (or column) space of $\hat{\w}$ is a subspace of the range (or column) space of $\w$. Next, we show $\text{range}({\w}) \subset \text{range}(\hat{\w})$. Since $A$ is invertible, we can write \revise{$\text{vec}(\w) = A^{-1}\cdot\text{vec}(\hat{\w})$}. For the same reason above, the range space of matrix $\w$ is the subset of the range space of matrix $\hat{\w}$. Then, we can complete the proof by combining both conclusions.
\end{proof}

Then, ignoring the constant terms, we can finalize the preconditioning reformation as follows:
\begin{equation}
\begin{aligned}
\min_{\hat{\w} \in \R^{T\times P}, b \in \R} &\quad
%\hat{f}(\hat{\w},\b)=
\frac{1}{2}\left\|\text{vec}(\hat{\w})\right\|^2 - (\1b +\y_{\Omega})^\top X_{\hat{\Omega}}^\top(X_{\Omega}X_{\Omega}^\top)^{-{1\over 2}}\text{vec}(\hat{\w}) + {1\over 2}\|\1b + \y_{\Omega}\|^2
\\
& + \frac{\lambda}{2}\left\|\min\left(\0, X_{\bar{\Omega}}^\top(X_{\Omega}X_{\Omega}^\top)^{-{1\over 2}}\text{vec}(\hat{\w}) - \1 b - \y_{\bar{\Omega}}\right)\right\|^2 \\
\text{s.t.}& \quad \text{rank}(\hat{\w}) \leq r.
\end{aligned}
\label{eq:main_3}
\end{equation}

%\tcx{shouldn't there be a ``b'' term for the first part? (Ji: I agree this comment. We should change it and the followed Algorithm 1.) }

Now we can apply PGD to solve \eqref{eq:main_3}. As summarized in Algorithm~\ref{alg:pgd}, \revise{first, we run gradient descent with a fixed step size $\eta$ on the smooth part (objective function); and then, solve the proximal mapping (projection) with Singular Value Decomposition (SVD) and recover $\w$ by $\text{vec}(\w) = (X_{\Omega}X_{\Omega}^\top)^{-{1\over 2}} \text{vec}(\hat{\w})$. Note that, in Algorithm~\ref{alg:pgd}, the criterion of convergence is met when the decrease rate of the objective function value is less than 0.0001 in an iteration.}

\begin{algorithm}                      % enter the algorithm environment
\caption{Projected Gradient Descent}          % give the algorithm a caption
\label{alg:pgd}                           % and a label for \ref{} commands later in the document
\begin{algorithmic}                    % enter the algorithmic environment
    \REQUIRE $X_{\Omega}$, $X_{\bar{\Omega}}$, $\lambda$, $r$ and $\eta$
    \ENSURE $\w$, $b$
    \STATE $A = (X_{\Omega}X_{\Omega}^\top)^{-{1\over 2}}$
    \WHILE{not converge}
    \STATE $\nabla_{\hat{\w}} = \text{vec}(\hat{\w}) - (\1 b +\y_{\Omega})^\top X_{\Omega}^\top A+ \lambda AX_{\bar{\Omega}} \min(\0,X_{\bar{\Omega}}^\top A\cdot
    \text{vec}(\hat{\w}) - b - \y_{\bar{\Omega}})$
    \STATE $\nabla_{b} = \1 (b+1) + \y_{\Omega} + \lambda \min(\0,X_{\bar{\Omega}}^\top A\cdot
    \text{vec}(\hat{\w}) - b - \y_{\bar{\Omega}})$
    \STATE $\text{vec}(\hat{\w}) = \text{vec}(\hat{\w}) - \eta\nabla_{\hat{\w}}$
    \STATE $b = b - \eta\nabla_{b}$
    \STATE $[U, \Sigma, V] = \text{svd}(\hat{\w})$
%    \STATE \% project $\hat{\w}$ to the set of rank $r$ matrix
    \STATE $\hat{\w} = U_{r} \Sigma_r V_r$
    \ENDWHILE
    \STATE $\text{vec}(\w) = A\cdot \text{vec}(\hat{\w})$
\end{algorithmic}
\end{algorithm}

One remaining issue in the implementation of the PGD approach is the selection of the step size $\eta$. It can be chosen as a sufficiently small value, or the AMIGO rule~\cite{aastrom2004revisiting} can be adopted to dynamically decide the value of $\eta$. No matter which approach is used, it has been shown in \cite{nesterov2004introductory} that the key issue is to ensure that the objective function is decreasing iteratively, which is not a difficult task empirically.

\subsection{Bounded Matrix Completion (BMC)}
In this section, we will develop a matrix completion approach to mitigate the missing data issue in our problem. Missing data has been ubiquitous in many healthcare applications. We found no exception in our study, since the data is collected by patients and for many reasons the patient compliance in data collection has always been a challenge. We propose to use the matrix completion approach to fill in the missing values, rather than using simple heuristic approaches such as using mean values to replace the missing values. Our rationale is, again, the correlations in the spatial-temporal matrix essentially imply that the spatial-temporal matrix is low dimensional. For instance, some variables that are used to measure the wound could be correlated, such as the color and temperature of the wound. Due to this reason, by using their correlations, we could recover many missing values as long as we have observed some values in the matrix. To exploit this idea, we assume that there exist a few basis vectors that span the daily observation vector of each individual (i.e., in the same spirit of many dimensionality reduction methods). Let $\mathcal{T}_{n,t}$ denote the $t$-th day observation vector of the $n$-th individual. We assume that $\mathcal{T}_{n,t}$ can be represented as a linear combination of the basis vectors $\{U_1, U_2, \cdots, U_r\}$:
\[
\mathcal{T}_{n,t} \approx \sum_{i=1}^r \alpha_i U_i.
\]

Given observed measurements in $\mathcal{T}_{n,t}$, the optimal weights $\{\alpha_1, \alpha_2, \cdots, \alpha_r\}$ can be obtained by minimizing the difference between the corresponding elements in $\sum_{i=1}^r \alpha_i U_i$ with the observed measurements in $\mathcal{T}_{n,t}$. In other words, we seek to identify the optimal weights that can best match the projected vector $\sum_{i=1}^r \alpha_i U_i$ with the observed incomplete vector $\mathcal{T}_{n,t}$. This rationale leads to the following formulation:
%, which is motivated by the assumption that the values of all attributes at each day are generated from the limited number of patterns.
%
%For any matrix
%$\mathcal{T} \in \mathbb{R}^{N\times T\times P}$ with index set $\Phi_{\mathcal{T}}$, define
%\[
%    \mathcal{T}^0_{n,t,p}=
%\begin{cases}
%    \mathcal{T}_{n,t,p},& \text{if } (n,t,p)\notin \Phi_{\mathcal{T}}\\
%    \mu_p,              & \text{otherwise}\\
%\end{cases}
%\]
%where $\mu_p$ is the mean value of the $p$-th attribute.\\
%Then, from $(\mathcal{T}^0,\Phi_{\mathcal{T}})$ we can generate the merged data
%matrix $(X_d^0,\Phi)$. The singular-value decomposition (SVD) of $X$
%gives
%\begin{align*}
%    X = U\Sigma^k {V}^{\top}.
%\end{align*}
\begin{align*}
    \min_{X_{\Phi},M} \quad & ||X-M||^2_{F} \\
    \text{s.t.} \quad & \text{rank}(M) \leq r \\
    & X_{i,j} \in [l_j, u_j] \quad \forall (i,j) \in \Phi
\end{align*}
where $M$ encodes the information of $\{U_1, U_2, \cdots, U_r\}$ and $\Phi$ denotes the set of the \revise{missing values. %$\{U_1, U_2, \cdots, U_r\}$ can be easily obtained by conducting a SVD:
%\[
%M = U\Sigma V^\top.
%\]
%The algorithm iteratively updates $M$ and $X_{\Phi}$ by
%\begin{align*}
%    M^k &= U_r^k \Sigma_r^k {V_r^k}^{\top} \\
%    X_{i,j}^{k+1} &= \proj_{[l_j,u_j]} (M^k_{i,j}),\quad
%    \forall(i,j)\in\Phi
%\end{align*}
Note that, in this bounded low-rank matrix completion problem, we further restrict the missing elements to be within a certain range $[l_j, u_j]$ for the corresponding variable $j$. This is to ensure that the missing values won't be filled in with unreasonable values. Actually, the lower bound and the upper bound can be decided based on the observed data to ensure the clinical validity of the estimated values:
\begin{align*}
    l_{j} :=& \min_{k\in \{i:~(i,j)\notin \Phi\}} X_{k,j}; \\
    u_{j} :=& \max_{k\in \{i:~(i,j)\notin \Phi\}} X_{k,j}.%\\
    %\text{for all}\quad & (k,j) \in \Phi^c
\end{align*}}
\revise{Note that these bounds are obtained from the observed data in $\Phi^c$, the complement of $\Phi$. }
%where $\Phi^c$ denotes the complement of $\Phi$.
%We can prove that there exists an $\bar{X}$, such that
%\begin{align*}
%    \lim_{n\to\infty} X^n = \bar{X}
%\end{align*}
%which is the completed matrix we are going to use in the next section.

To solve the BMC problem, we apply the coordinate descent algorithm that iteratively \revise{solves the optimization by alternately updating} $X_{\Phi}$ or $M$ while fixing the other one. Since the value of the objective function decreases iteratively, this algorithm is guaranteed to converge. It is worthy of mentioning that, there is a closed-form solution in each step to solve for $M$ or $X_{\Phi}$. We summarize the BMC algorithm in Algorithm~\ref{alg:bmc}.
\begin{algorithm}                      % enter the algorithm environment
\caption{Bounded Matrix Completion (BMC)}          % give the algorithm a caption
\label{alg:bmc}                           % and a label for \ref{} commands later in the document
\begin{algorithmic}                    % enter the algorithmic environment
    \REQUIRE \revise{$X_{\Phi^c}$}, $r$, $l_j$ and $u_j$
    \ENSURE \revise{$X_{\Phi}$}
    \STATE \revise{Initialize $X_{\Phi}$}
    \WHILE{not converge}
    \STATE $[U, \Sigma, V] = \text{svd}(X)$
%    \STATE \% project $M$ to the set of rank $r$ matrix
    \STATE $M = U_{r} \Sigma_r V_r$
%    \STATE \% project $M_{ij}$ onto the range $[l_j, u_j]$
    \STATE $X_{ij} = \max(l_j, \min(u_j, M_{ij}))\quad \forall (i,j)\in$ \revise{$\Phi$}
    \ENDWHILE
\end{algorithmic}
\end{algorithm}

\subsection{Missing Value Estimation for the Testing Data}

As shown in the previous section, the BMC formulation is used to fill in the missing values in the training data. The rationale is to learn a few basis vectors $U_i$'s, and then, for each incomplete observation vector $\mathcal{T}_{n,t}$, the algorithm seeks to identify the optimal weights that can best match the projected vector $\sum_{i=1}^r \alpha_i U_i$ \revise{with $\mathcal{T}_{n,t}$, where $U =[U_1, U_2, \ldots, U_r]$ is obtained by applying SVD to the learned low-rank imputation matrix: $M= U\Sigma V^\top$.} Then, the missing values in $\mathcal{T}_{n,t}$ can be filled in by the corresponding elements in $\sum_{i=1}^r \alpha_i U_i$. Thus, it is straightforward to use the results from the BMC formulation learned from the training data to fill in the missing values in the testing data as well. Denote the coming testing data as $z$. For the new testing data $z$, its missing elements indexed by $S$ can be estimated by solving the following convex optimization problem:
\begin{align*}
\min_{\alpha, z_{S}}\quad & \|z-U\alpha\|^2
\\
\text{s.t.}\quad & z_j \in [l_j, u_j] \quad \forall~j \in S.
\end{align*}

Then, the best estimation can be found by updating $z_{S}$ and $\alpha$ iteratively by following the formulas shown below:
\[
\alpha = U^\top z,\quad z_j = \min(U_j, \max(l_j, U_{j\cdot}\alpha)),\quad \forall j \in S.
\]

\revise{Note that, many of the existing missing data imputation methods have been built on different assumptions of the mechanism of why data is missing, such as the Missing Completely At Random (MCAR) and the Missing At Random (MAR). While this may lead to more tailored missing data imputation methods, here, we have difficult in identifying and validating the missing data mechanism in our study. Hence, we pursue this low-rank model that places less strict assumptions on the missing data mechanism, which leads to the development of the low rank based estimation methods. Also, the low rank based estimation method can be applied to scenarios where a large portion of elements (e.g., 50\% elements) are missing. }

\section{Experiments}
\label{S:4}

\subsection{Study population}
\begin{table}[hp]
	\centering
	\fontsize{10}{12}\selectfont
	\caption{Baseline data from patient cohorts with and without inpatient SSI}
	\begin{tabular}{L{0.3\textwidth} L{0.25\textwidth} L{0.25\textwidth} L{0.08\textwidth}}
		\hline
		& \begin{tabular}{@{}c@{}}\textbf{Without SSI} \\ (N=684; 80.6\%)\end{tabular} & \begin{tabular}{@{}c@{}}\textbf{\revise{With} SSI} \\ (N=167; 19.4\%)\end{tabular} & \textbf{p-value} \\
		\textbf{Patient factors} & & \\
		Age, mean, {[}CI{]}, years & 56.18 {[}55.09-57.27{]} & 57.48 {[}55.45-59.51{]} & 0.29 \\
		Male sex, N (\%) {[}CI{]} & 247 (36.1) {[}0.33-0.40{]} & 62 (37.1) {[}0.30-0.45{]} & 0.81 \\ \hline
		\textbf{Procedure-related} & & \\
		Type of operation & N (\%) & N (\%) & $<$0.001\\
		\quad abdominal wall & 44 (6.4) & 3 (1.8) \\
		\quad gastroduodenum & 27 (3.9) & 4 (2.4) \\
		\quad gall bladder/bile duct & 31 (4.5) & 4 (2.4) \\
		\quad liver & 101 (14.8) & 19 (11.4) \\
		\quad spleen/adrenal gland & 29 (4.2) & 4 (2.4) \\
		\quad small bowel & 35 (5.1) & 18 (10.8) \\
		\quad kidney & 179 (26.2) & 23 (13.8) \\
		\quad vascular & 50 (7.3) & 6 (3.6) \\
		\quad esophagus & 75 (11.0) & 25 (15.0) \\
		\quad large bowel & 69 (10.1) & 35 (21.0) \\
		\quad pancreas & 44 (6.4) & 26 (15.6) \\ \hline
		\textbf{Risk factors} & & \\
		& N (\%) {[}CI{]} & N (\%) {[}CI{]} \\
		Smoking & 287 (42.0) {[}0.38-0.46{]} & 59 (35.3) {[}0.28-0.43{]} & 0.12\\
		Diabetes mellitus & 83 (12.2) {[}0.10-0.15{]} & 21 (12.6) {[}0.08-0.19{]} & 0.88\\
		Chronic lung disease & 58 (8.5) {[}0.07-0.11{]} & 22 (13.2) {[}0.08-0.19{]} & 0.063 \\
		Systemic corticosteroid & 79 (11.6) {[}0.09-0.14{]} & 25 (15.0) {[}0.10-0.21{]} & 0.23 \\
		Chemotherapy & 46 (6.7) {[}0.05-0.09{]} & 12 (7.2) {[}0.04-0.12{]} & 0.83 \\
		Radiotherapy & 12 (1.8) {[}0.01-0.03{]} & 3 (1.8) {[}0.00-0.05{]} & 0.97 \\
		Ascites present & 16 (2.3) {[}0.01-0.04{]} & 10 (6.0) {[}0.03-0.11{]} & 0.014 \\
		Infection (non-SSI) & 75 (11.0) {[}0.09-0.14{]} & 14 (8.4) {[}0.05-0.14{]} & 0.33 \\
		Alcohol use & 311 (47.1) {[}0.43-0.51{]} & 70 (45.5) {[}0.37-0.54{]} & 0.71 \\
		%& mean {[}CI{]} units/week & mean {[}CI{]} units/week \\
		Alcohol quantity & 4.49 {[}3.85-5.12{]} & 5.24 (3.54-6.94) & 0.34 \\
		Body Mass Index & & & 0.65 \\
		\quad Underweight & 19 (2.8) & 6 (3.6) \\
		\quad Normal & 317 (45.2) & 80 (47.9) \\
		\quad Overweight & 220 (32.2) & 59 (35.3) \\
		\quad Class 1 obesity & 80 (11.7) & 20 (12.0) \\
		\quad Class 2 obesity & 20 (2.9) & 8 (4.8) \\
		\quad Class 3 obesity & 8 (1.2) & 2 (1.2) \\ \hline
	\end{tabular}
	\label{baseline_demo}
\end{table}

\begin{figure*}[tp]
	\centering
	\includegraphics[width = 0.7\textwidth]{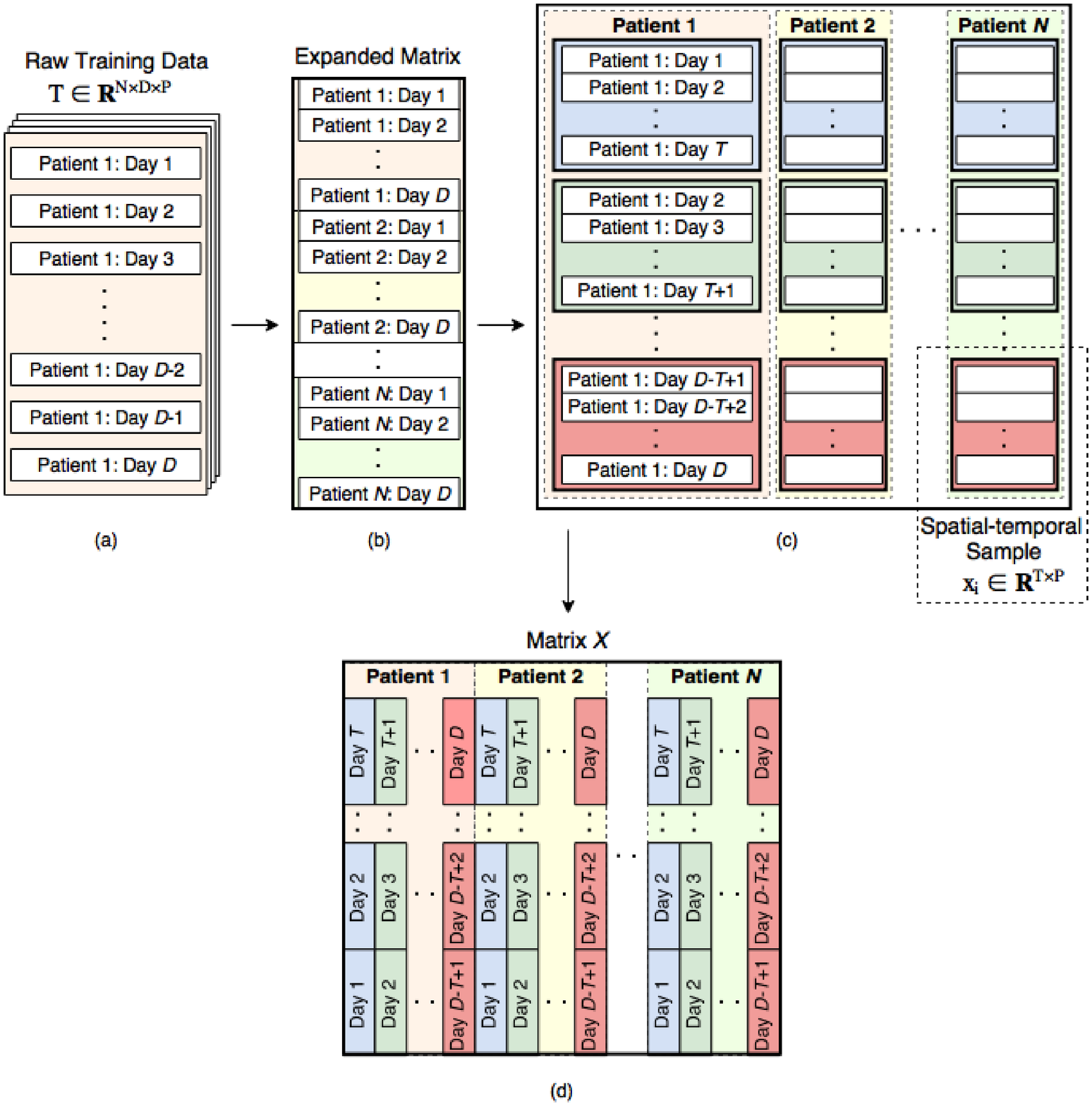}
	\caption{The data structure}
	\label{fig:gen}
\end{figure*}

A prospective cohort study of 1,000 open abdominal surgery patients was conducted at a 1200-bed academic teaching hospital in the Netherlands, described previously \cite{Ramshorst2013}. Patients who didn't undergo surgery (n=33) or with $<$ 2 days of wound observations (n=107) were excluded from analysis, leaving 860 patients in total. Subjects were prospectively assessed for CDC-defined superficial, deep and/or organ space infections \cite{Horan1992}. Superficial SSI are infections that occur within 30 days after the operation and involve only skin or subcutaneous tissues of the incision and at least one of the following: 1) purulent drainage from the incision, with or without confirmation by laboratory tests; 2) organisms isolated from an aseptically obtained culture or aspiration of the incision; 3) at least one of the following signs or symptoms: pain or tenderness, localized swelling, redness, or heat and superficial incision opened by the surgeon; and 4) diagnosis of superficial SSI by the surgeon or attending physician. Stitch abscesses and infected burn wounds are excluded. Deep incisional SSIs are infections occurring within 30 days of surgery, involving the deep soft tissues (muscle and fascia) and at least one of the following: 1) purulent drainage from the deep incision; 2) incision spontaneously dehisces or is opened by the surgeon when the patient has at least one of the following: fever (38°C), localized pain, or tenderness; 3) an abscess or other evidence of infection involving the deep incision is found on examination, reoperation, or histopathologic or radiology examination; and 4) deep SSI is diagnosed by the surgeon or attending physician. Finally, an organ space infection is an infection occurring within 30 days of surgery; the infection appears related to the surgery and involves any part of the anatomy (organs or spaces), other than the incision and at least one of the following: 1) purulent drainage from drain; 2) organisms isolated from an aseptically obtained culture of fluid or tissue; 3) diagnosis by the surgeon or attending physician; and 4) an abscess or other evidence of infection involving the organ space found on direct examination, reoperation, or radiologic examination. Patients' characteristics are reported in Table \ref{baseline_demo} for both SSI and non-SSI patients\revise{, where the p-value is computed by the t-test for comparing the two groups}.

\subsection{Data collection}
Subjects in the dataset were examined daily, using a previously described protocol \cite{Ramshorst2014}, from post-operative day 2 until discharge or 21 days, whichever was earlier. Follow-up was performed at 30 days through clinic visit, phone, or letter to ascertain post-discharge infections. Table \ref{conti_data_demo} shows the dynamic wound data that was collected, including definitions of categorical wound score variables \cite{Ramshorst2014}. For analysis purposes, we defined the SSI group as having any of the 3 types of SSI due to the small numbers of deep and organ-space infections. In addition, though a patient may have developed multiple types of SSI during the observation period, we only include their first infection in this analysis. The non-SSI group was defined as having not developed any kind of infection, but may have had a post-discharge infection. \revise{For patients who didn't develop SSI during the observational period, their days of observation are 21. For patients who developed SSI during the observational period, the summary statistics of their observational time is shown in Table \ref{complete_ssi}.}

\begin{table}[hp]
\centering
\begin{tabular}{c|c|c|c|c|c}
Minimum & 1st Quantile & Median & Mean  & 3rd Quantile & Maximum \\ \hline
3    & 5       & 7      & 8.244 & 11      & 21
\end{tabular}
\caption{Days of post-operative observation for patients developing in-hospital SSI}
\label{complete_ssi}
\end{table}

Further, Figure ~\ref{fig:gen} schematically illustrates the process by which we translated the original data into the data structure that will be used by the proposed machine learning system. The basic idea is to use a sliding window to segment each individual's time series measurements into blocks of equal size $T$.

% \begin{figure*}[ht]
%     \centering
%     \includegraphics[width=.8\textwidth]{pie}
%     \caption{SSI groups}
%     \label{fig:pie}
% \end{figure*}

\fontsize{10}{12}\selectfont
\begin{longtable}{ll L{0.5\textwidth}}
\caption{Repeated data collected} \\
\endfirsthead
\multicolumn{3}{l}
{\tablename\ \thetable\ -- \textit{Continued from previous page}} \\
\endhead
\multicolumn{3}{r}{\textit{Continued on next page}} \\
\endfoot
\endlastfoot
\hline
Variable & Scale & Details \\
\hline
\textbf{Primary wound variables} & & \\
Induration amount (mm) & 0 & $>$5 mm \\
& 1 & 3-4 mm \\
& 2 & 1-2 mm \\
& 3 & 0 mm \\
\hline
Wound edge distance (mm) & 0 & 0 mm \\
& 1 & 1-2 mm \\
& 2 & 3-5 mm \\
& 3 & 6-10 mm \\
& 4 & 11+ mm \\
\hline
Slough/necrosis type & 0 & none visible \\
& 1 & white/grey nonviable tissue \\
& 2 & loosely adherent yellow slough \\
& 3 & adherent, soft, black eschar \\
& 4 & firmly adherent, hard, black eschar \\
\hline
Slough/necrosis amount & 0 & None visible \\
& 1 & $<$25\% of wound bed covered \\
& 2 & 25 to 50\% of wound covered \\
& 3 & $>$50\% and $<$75\% of wound covered \\
& 4 & 75 to 100\% of wound covered \\
\hline
Granulation/epithelialization score & 0 & Skin intact \\
& 1 & 75 to 100\% of wound filled \\
& 2 & 25 to 75\% of wound filled \\
& 3 & $<$25\% of wound filled \\
& 4 & no granulation or epithelialization present \\
\hline
Exudate type & 0 & none or bloody \\
& 1 & serosanguineous: thin, watery, pale red/pink \\
& 2 & serous: thin, watery, clear \\
& 3 & purulent: thin or thick, opaque, tan/yellow \\
& 4 & foul purulent: thick, opaque, yellow/green with odor \\
\hline
Exudate amount & 0 & none (tissue is dry) \\
& 1 & scant (non measurable amount) \\
& 2 & small (exudate spread over wound, gauzes 25\% wet) \\
& 3 & moderate (exudate irregularly spread over wound) \\
& 4 & large (large amount, widespread, gauzes $>$75\% wet) \\
\hline
Wound edge color & 0 & pink or normal for ethnic group \\
& 1 & bright red and/or blanches to touch \\
& 2 & white or gray pallor or hypopigmented \\
& 3 & dark red or purple and/or nonblanchable \\
& 4 & black or hyperpigmented \\
\hline
Temperature $10\,^{\circ}\mathrm{C}$ &  & Wound \\
& & 1 cm from wound edge (left/right) \\
& & 3 cm from wound edge (left/right) \\
& & 5 cm from wound edge (left/right) \\
\hline
Wound malodor & & Yes/no \\
\hline
\textbf{Other wound variables} & & \\
Hematoma* & & Yes/no \\
wound mass palpable* & & Yes/no \\
seroma* & & Yes/no \\
wound culture* & & Yes/no \\
visual analogue pain scale & 1-100 & Wound pain \\
wound length (cm) & \revise{count} & \dots cm \\
\hline
\textbf{Vital signs} & & \\
heart rate & \revise{count} & \dots bpm \\
diastolic RR & \revise{count} & \dots mmHg \\
systolic RR & \revise{count} & \dots mmHg \\
tympanic temperature $10\,^{\circ}\mathrm{C}$ & \revise{count} & \dots $10\,^{\circ}\mathrm{C}$ \\
\hline
\textbf{Other observations} & & \\
cough* & & Yes/no \\
productive cough* & & Yes/no \\
vomiting* & & Yes/no \\
ventilator* & & Yes/no \\
antibiotics* & & Yes/no \\
reoperation* & & Yes/no \\
nasogastric tube* & & Yes/no \\
Suspicion of ileus* & & Yes/no \\
serial operation number & \revise{count} & \\
\hline
* in previous 24 hours
\label{conti_data_demo}
\end{longtable}
\normalsize % reset to default font size

% \begin{figure*}[h]
%     \centering
%     \includegraphics[width=1.0\textwidth]{act}
%     \caption{Actual}
%     \label{fig:act}
% \end{figure*}

% \begin{figure*}[h]
%     \centering
%     \includegraphics[width=1.0\textwidth]{scatter}
%     \caption{Predictive / Actual}
%     \label{fig:scatter}
% \end{figure*}

\subsection{Parameter tuning and validation}
Cross validation is adopted to tune the parameters in each model that we will build in this study. %In the validation of the model, we only use the complete samples due to the insufficient knowledge of the outcomes for the censored samples.
\revise{Specifically, we randomly split the samples into multiple sets with equal size. In our implementation, we take a training-testing ratio of 4:1. Then we  choose 4 sets as the training data, and the one left as the testing data. For each proposed model, we repeat the training-testing procedure 5 times, to ensure that every set is used as the testing data once.} We use the mean absolute prediction error (MAE), i.e.,  $\sum_{i\in\Omega} {|\Omega |}^{-1} |\langle \x_i,~\w\rangle+b-y_i|$, to evaluate the prediction performance of each model in predicting the onset time of the individuals.

% \begin{figure*}[h]
%     \centering
%     \includegraphics[width=1.0\textwidth]{rmse}
%     \caption{RMSE / lambda}
%     \label{fig:rmse}
% \end{figure*}

% \begin{figure*}[h]
%     \centering
%     \includegraphics[width=1.0\textwidth]{imp_rmse}
%     \caption{RMSE / Imputation}
%     \label{fig:imprmse}
% \end{figure*}

\begin{figure*}[h]
    \centering
    \includegraphics[width=0.8\textwidth]{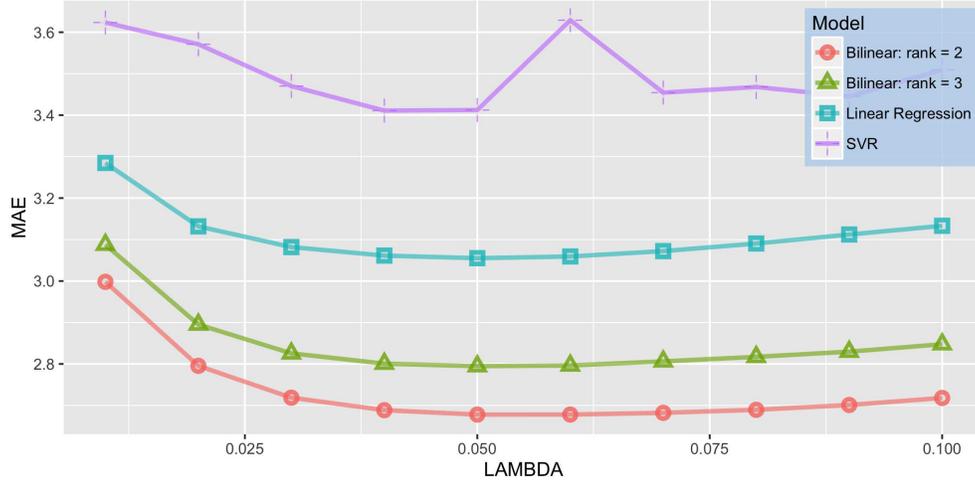}
    \caption{The MAE of different methods across different $\lambda$}
    \label{fig:mae_comparison}
\end{figure*}

\begin{figure*}[h]
    \centering
    \includegraphics[width=0.8\textwidth]{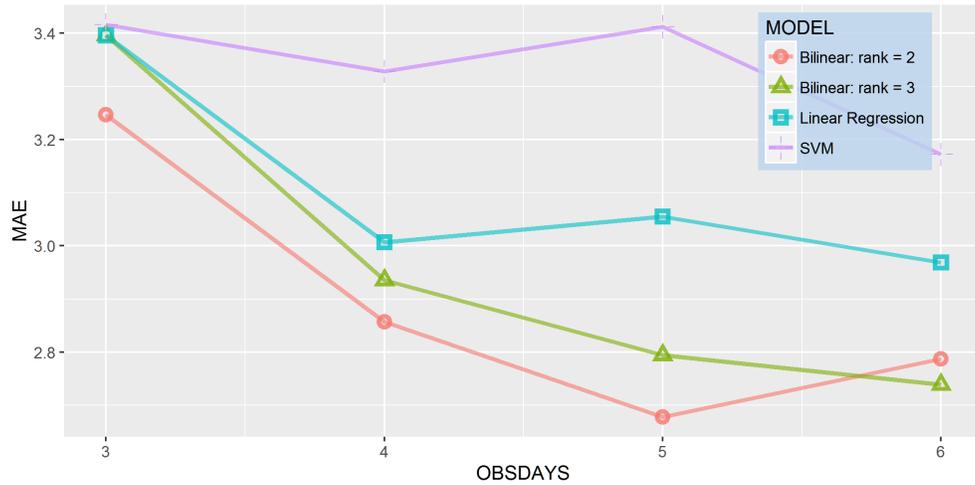}
    \caption{The MAE of different models across different duration of observation}
    \label{fig:mae_days}
\end{figure*}

\subsection{Comparison with other models}
We will compare our proposed learning system with closely related benchmark approaches. The first one is the classic regression model \cite{LR} that assumes linear and additive predictive effects of the predictors on a continuous outcome variable. It has been used in many risk prediction models in a wide range of applications. The second one is the Support Vector Regression (SVR) model \cite{SVR}. It is a more robust model than the classic regression model by using a soft margin loss function, representing the state-of-the-art performance. The performance comparison of these models are shown in Figures \ref{fig:mae_comparison} and \ref{fig:mae_days}. Note that, Figure \ref{fig:mae_comparison} is to compare the methods under different values of $\lambda$ which essentially reflects the weight of the censored samples in the learning formulation. Since we have an unbalanced dataset, using weights to balance the contribution of both the complete and censored samples has been a common approach in the use of the existing methods. From both Figures \ref{fig:mae_comparison} and \ref{fig:mae_days}, it can be observed that our methods perform significantly better than both benchmark methods across all the comparisons. Also, it seems that with rank = 2, our method performs better than its companion that has rank = 3, indicating that for this dataset, rank = 3 might result in overfitting. \revise{Note that, in generating this result, we didn't include the selection of the rank into the cross-validation procedure. For a given rank, we adopted cross-validation to select for the remaining parameters. Actually, \revise{later in Section 4.6, we will show that the ``full-scale'' cross-validation study that includes the selection of the rank actually selected rank = 2 as the optimal model, which indicates that the cross-validation can help control the model complexity and reduce the risk of over-fitting.}}

%later in Section 4.6, Table \tcq{4}, we have shown that in a full-scale cross-validation study that includes the selection of the rank, it actually selected rank = 2 as the optimal model that indicates that the cross-validation could help reduce risk of over-fitting.}

\subsection{Evaluation of the matrix completion method for missing data imputation}
We further evaluate the performance of the matrix completion method for missing data imputation in comparison with other methods. For example, for the ``column mean'' method, to fill in the missing value on feature $j$ of individual $i$ at a certain time point, the mean value of feature $j$ of all the individuals across all the time points can be used based on the observed measurements. \revise{Other methods we used for comparison include the K-Nearest-Neighbor (KNN) method that has been widely used for missing data imputation \cite{kim2004reuse,Keller1985} and a recent Bayesian approach \cite{mipackage}. We train the models using different missing data imputation methods under different values of $\lambda$ and report the results as shown in Figure \ref{fig:impmae}.} It clearly shows that our low-rank matrix completion method could lead to superior performance on missing data imputation, providing better prediction performance of the model.

\begin{figure*}[h]
    \centering
    \includegraphics[width=0.8\textwidth]{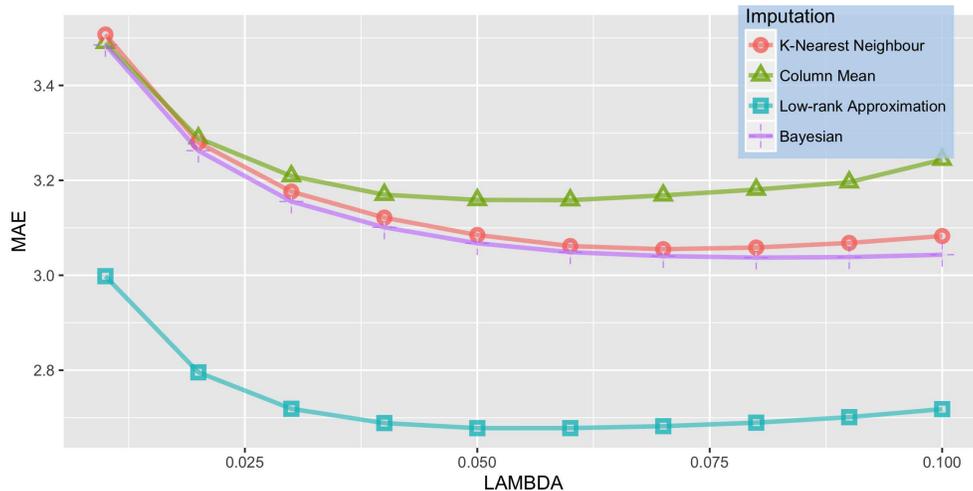}
    \caption{The MAE of different models that use matrix completion and other procedures for missing data imputation, respectively}
    \label{fig:impmae}
\end{figure*}

\subsection{Optimal configuration of the proposed learning system}
To identify the optimal configuration of the proposed learning system, we build a sequence of models to investigate the performances of different configurations. Critical parameters include the duration of the dynamic data we could use for building the prediction model, the rank of the model parameters $\w$, and the value of $\lambda$. Among these parameters, probably the duration of the dynamic data is the most important parameter. Apparently, there is a trade-off between accuracy and cost. With longer duration, more information could be accumulated which will lead to better prediction accuracy. On the other hand, it will result in higher cost, not only in data collection, but also in the healthcare cost due to the potential delay of identifying the patients who will develop SSI soon. Therefore, it is a practical and important decision to find a reasonable length for observation and prediction. Here, \revise{based on the cross validation with the training set ratio set to be 4:1,} the prediction performances of the models are shown in Table \ref{M3}. This demonstrates that the best model could be built if the duration of the dynamic wound data is 5-days and the rank of $w$ is 2. A worse performance was generated when the length of observation is 6. This is probably because that the sample size for training the model drops significantly when the length of observation increases. \revise{In Figure \ref{fig:vars}, we also present the important variables of the best model in Table \ref{M3} with their coefficients in decreasing order.} We further show the distribution of the predicted onset day for both the complete and censored populations using our best model in Figure  \ref{fig:dist_pred}, and demonstrate that our model can effectively separate the two groups.

\begin{figure}[hp]
\centering
\includegraphics[width=1\textwidth]{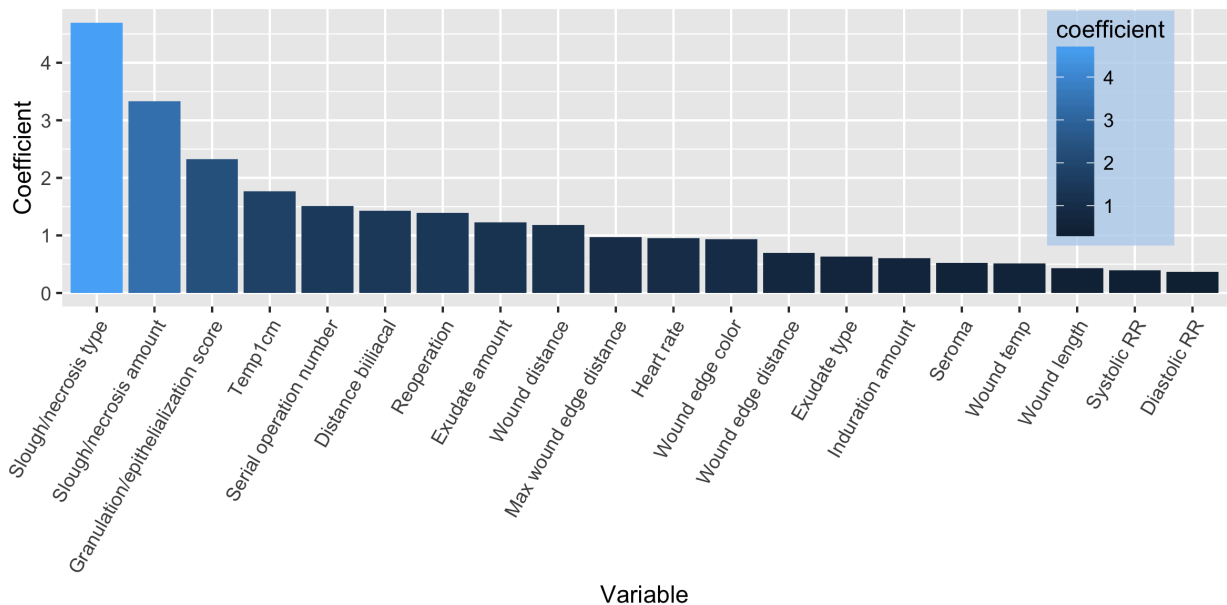}
\caption{Selected important variables with their coefficients}
\label{fig:vars}
\end{figure}

% \begin{figure*}[h]
%     \centering
%     \includegraphics[width=1.0\textwidth]{days_mae}
%     \caption{The MAE of different days of observation}
%     \label{fig:days_mae}
% \end{figure*}

\begin{table}[]
\centering
\begin{tabular}{ccccc}
\hline
 &  & \multicolumn{3}{c}{$\lambda$-value} \\ \hline
\multicolumn{1}{c|}{Duration} & \multicolumn{1}{c|}{rank($\w$)} & \multicolumn{1}{c|}{0.01} & \multicolumn{1}{c|}{0.05} & 0.10 \\ \hline
\multicolumn{1}{c|}{\multirow{2}{*}{3}} & \multicolumn{1}{c|}{2} & \multicolumn{1}{c|}{3.720192} & \multicolumn{1}{c|}{3.246942} & 3.206233 \\
\multicolumn{1}{c|}{} & \multicolumn{1}{c|}{3} & \multicolumn{1}{c|}{3.817237} & \multicolumn{1}{c|}{3.395871} & 3.364114 \\ \hline
\multicolumn{1}{c|}{\multirow{2}{*}{4}} & \multicolumn{1}{c|}{2} & \multicolumn{1}{c|}{3.260273} & \multicolumn{1}{c|}{2.857273} & 2.863779 \\
\multicolumn{1}{c|}{} & \multicolumn{1}{c|}{3} & \multicolumn{1}{c|}{3.334785} & \multicolumn{1}{c|}{2.935544} & 2.962186 \\ \hline
\multicolumn{1}{c|}{\multirow{2}{*}{5}} & \multicolumn{1}{c|}{2} & \multicolumn{1}{c|}{2.998079} & \multicolumn{1}{c|}{\bf2.677925} & 2.718011 \\
\multicolumn{1}{c|}{} & \multicolumn{1}{c|}{3} & \multicolumn{1}{c|}{3.087649} & \multicolumn{1}{c|}{2.794683} & 2.847655 \\ \hline
\multicolumn{1}{c|}{\multirow{2}{*}{6}} & \multicolumn{1}{c|}{2} & \multicolumn{1}{c|}{3.085667} & \multicolumn{1}{c|}{2.787124} & 2.779734 \\
\multicolumn{1}{c|}{} & \multicolumn{1}{c|}{3} & \multicolumn{1}{c|}{3.080365} & \multicolumn{1}{c|}{2.739127} & 2.487222 \\ \hline
\end{tabular}
\caption{Prediction performance of different model configurations}
\label{M3}
\vspace{-5mm}
\end{table}

\begin{figure*}[h]
    \centering
    \includegraphics[width=0.8\textwidth]{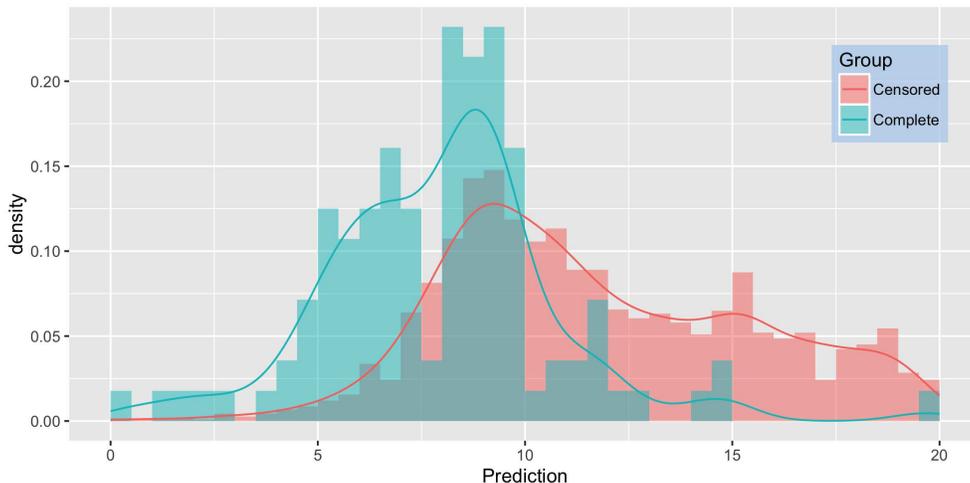}
    \caption{Distribution of the predicted onset day for both the complete and censored samples}
    \label{fig:dist_pred}
\end{figure*}

% \subsection{Insights for SSI prevention}
% Our experiments show it is possible to predict the onset time of SSI using the dynamic data of individuals. This will definitely help surgeons and healthcare providers to identify the patients who will develop SSI sooner than others, leading to better allocation of healthcare resources, prioritization of the patients to close monitoring and follow-up visits, etc. On the other hand, it is of interest to investigate the important features that have been critical for the prediction model. The TOP 10 variables that show up, according to the L2 norm of their weights in the prediction model, are presented in Table \ref{top10}. While it is hardly to establish causality between these variables with the SSI outcome, we notice that most of these variables are actually controllable. The controllability of the variables has implication that we could control or reduce the risk of SSI by controlling these variables.

% \begin{table}[h]
% \centering
% \caption{Top 10 features}
% \begin{tabular}{ll}
% \toprule
% {} & Feature names \\ \midrule
% 1         & ``Abdominal Circumference"           \\
% 2         & ``XiphoidosPubis"        \\
% 3         & ``distbiiliacal"         \\
% 4         & ``1cm\_wound"            \\
% 5         & ``wounddistance"         \\
% 6         & ``RRDIA"                 \\
% 7         & ``RRSYS"                 \\
% 8         & ``maxwounddist"          \\
% 9         & ``pulse"                 \\
% 10        & ``woundcolour"           \\ \bottomrule
% \end{tabular}
% \label{top10}
% \end{table}

\section{Conclusions}
\label{S:5}
Accurate determination of the SSI risk of a particular patient is essential for deciding whether or not particular preventive strategies, i.e., enhanced prophylactic antibiotics, decolonization, or other emerging interventions should be adopted. Furthermore, uncovering developing SSI at the earliest possible time affords the opportunity for anticipatory intervention, preventing progression of infection and its morbid sequelae. It also holds great potential to reduce the healthcare cost since delayed diagnosis of post-discharge SSIs has significant financial and life quality costs, with more than half of patients who develop post-discharge SSI readmitted to the hospital \cite{Gibson2014}. In this paper, we investigated the problem of predicting the time to the onset of SSI using the dynamic wound data collected on individuals. The dynamic data of an individual, represented as the spatial-temporal matrix, include the continuous measurements of many wound-related variables. Many of the wound variables are applicable in a post-discharge setting, allowing monitoring of these wounds using automated image analysis to predict infections in real time. Thus, this study aligns well with the emerging mHealth tools that are developed to closely monitor SSI patients. While most existing SSI prediction models only use preoperative and operative variables, \cite{lawson2013risk,berger2013development,van2013surgical,ho2011differing,saunders2014improving}, to the best of our knowledge, we are the first team who developed a systematic treatment to utilize the dynamic wound data of an individual for SSI risk prediction. This study used a unique dataset, coming from a hospital in the Netherlands, which had a significantly longer average length of stay than in the US, allowing us to characterize a longer spectrum of the SSI progression process.

Our study is subject to limitations. First, we have only investigated the use of the dynamic wound data for SSI risk prediction. However, it has been found that many preoperative and operative variables are important in predicting the SSI risk. Knowledge of the potential impact of these predictors may help guide physicians with both preoperative and operative decision-makings, and help develop personalized post-discharge care plans and interventions as well. For instance, previous studies have demonstrated that poor nutrition is associated with increased rates of SSI, probably due to a depressed immune system \cite{Malone2002}. Thus, these physiologically impaired patients might benefit from timely and appropriate nutritional support. It has also been found that smokers have higher SSI risk than non-smokers, probably due to the altered perfusion mechanics at the surgical site. There are many other preoperative and operative factors such as age, surgery type, operative time, glucose, length of stay \cite{Haridas2008,Malone2002}, that could significantly increase our capability to predict the SSI risk if being combined with the dynamic wound data. Second, SSI was classified into 3 categories based on the depth of infection \cite{Horan1992}, which include superficial, deep, and organ space SSI. It might be possible that different kinds of SSI would have different risk profiles. Third, many variables that have been found predictive of SSI risk in the literature are not collected in our dataset. For instance, specifics of the operative procedure, such as timing of antibiotics, core body temperature, blood glucose level, are not currently available. On the other hand, we have not included the image-based wound variables, which theoretically can be extracted from the wound images that have been available in the dataset, into our prediction model. The challenge to extract those image-based wound variables from the wound image data includes the enormous variations in the parameters regarding how the images are acquired by the individual patients, i.e., with different angles and different light conditions, etc. Thus, this is also one of our future research directions to extract these image-based wound variables and incorporate them into our prediction model. Finally, although our dataset has longer observation duration than many other datasets studied in the literature, it is still a short period of time considering that the typical follow up period for SSI is thirty days after discharge. This interval of observation is more administratively derived than biologically derived, as SSI could potientilaly occur at longer periods after the surgical event itself.

In summary, in this study, we have developed a systematic treatment of the data challenges associated with the dynamic wound data. The model can predict not only yes-or-no, but also when the SSI will be developed. In addition,  this model has potential use in real-time mHealth monitoring systems that can incorporate patient-reported outcomes and automate image analysis to predict post-discharge SSI. In our future research, we will investigate how to integrate the dynamic wound data with some static risk factors such as some preoperative and operative factors to further improve the SSI risk prediction. We will also develop robust image feature extraction method to extract SSI-related features from the wound image data that is available in our dataset. Also, it is of interest to investigate the course of SSI development and evolution for different types of SSI and surgery procedures. \revise{We notice that previous works such as \cite{Toma2010} raised an interesting question from a conceptual point of view, which is, whether one could generate meaningful features or episodes from variables over time. It is a very interesting question to derive such features or patterns to better understand the characteristics of the underlying health condition. While the pathogenesis and natural history of SSI has been largely unknown, a reasonable assumption is that there is a progression process of SSI that leaves a trail of clinical patterns/features. Through this study, we demonstrate that using dynamic data could improve SSI  prediction. It will be our future work to thoroughly study the dynamics of the clinical signals and its connection/implication of the underlying disease progression process. Last but not least, note that we assumed the prediction model is linear. Despite the appearance as a linear model, proper transformations could be applied to extend the proposed linear model for characterizing potential nonlinearity. On the other hand, if the transformation is needed, we do need to base the linear models on interpretable and meaningful variable transformation. Besides the methodological issues for data fusion and predictive modeling, how to integrate the monitoring tool in real surgical settings is definitely a very important problem. We have studied some important aspects of this problem in our previous work \cite{Patrick2016}. Usability research is also part of our future research directions, e.g., in \cite{Patrick2016} we only focused on a unique post-acute surgical use case that may not perfectly generalize to medical or chronic care settings.}

\section*{Acknowledgments}
Chuyang Ke and Ji Liu are in part supported by the NSF grant CNS-1548078.

%hus, although the fundamental dynamics is unknown yet, t



%% References

%\section*{References}
\bibliographystyle{unsrt}
\bibliography{reference.bib}






% \end{thebibliography}

\end{document}